\documentclass{article}

\PassOptionsToPackage{numbers}{natbib}

\usepackage[preprint]{nips_2018}

\usepackage[utf8]{inputenc} 
\usepackage[T1]{fontenc}    
\usepackage{hyperref}       
\usepackage{url}           
\usepackage{booktabs}       
\usepackage{amsfonts}       
\usepackage{nicefrac}       
\usepackage{microtype}      
\usepackage{xspace}
\usepackage{amsmath}
\usepackage{amsthm}

\newtheorem{assumption}{Assumption}
\usepackage{algorithm}
\usepackage{algpseudocode}
\usepackage{algpascal}
\usepackage{comment}
\usepackage{times}

\newtheorem{theorem}{Theorem}
\newtheorem{proposition}{Proposition}
\newtheorem{lemma}{Lemma}

\newcommand{\real}{\mathbb{R}}

\newcommand{\II}[1]{\mathbb{I}_{\left\{#1\right\}}}
\newcommand{\PP}[1]{\mathbb{P}\left[#1\right]}
\newcommand{\E}{\mathbb{E}}
\newcommand{\EE}[1]{\mathbb{E}\left[#1\right]}

\newcommand{\EEcc}[2]{\mathbb{E}\left[\left.#1\right|#2\right]}

\newcommand{\ev}[1]{\left\{#1\right\}}
\newcommand{\pa}[1]{\left(#1\right)}
\newcommand{\bpa}[1]{\bigl(#1\bigr)}

\newcommand{\wh}{\widehat}
\newcommand{\wt}{\widetilde}
\newcommand{\transpose}{^\top}

\def\argmax{\mathop{\mbox{ arg\,max}}}

\newcommand{\ra}{\rightarrow}
\usepackage{todonotes}
\definecolor{PalePurp}{rgb}{0.66,0.57,0.66}

\newcommand{\ducb}{$d$-UCB\xspace}

\newcommand{\ducbv}{$d$-UCB$(V_0)$\xspace}
\newcommand{\klucb}{kl-UCB\xspace}
\newcommand{\irg}{$G(n, A)$\xspace}
\newcommand{\chlu}{$G(n, w)$\xspace}
\newcommand{\sbm}{$G(n, \alpha, K)$\xspace}
\newcommand{\chunglu}{Chung--Lu\xspace}
\newcommand{\GW}{W}
\newcommand{\TT}{\mathcal{T}}
\newcommand{\Sw}{\mathcal{S}}

\newcommand{\poi}{\text{Poi}}

\newcommand{\ber}{\text{Ber}}
\newcommand{\hmu}{\wh{\mu}}

\title{Online Influence Maximization \\with Local Observations}

\author{
  Julia Olkhovskaya \\
  Universitat Pompeu Fabra\\
  Barcelona, Spain\\
  \texttt{julia.olkhovskaya@gmail.com} \\
  \And
  Gergely Neu \\
  Universitat Pompeu Fabra\\
  Barcelona, Spain\\
  \texttt{gergely.neu@gmail.com} \\
  \And
  G\'abor Lugosi\\
  ICREA \& Universitat Pompeu Fabra\\
  Barcelona, Spain\\
  \texttt{gabor.lugosi@gmail.com} \\
}

\usepackage{selectp}

\begin{document}

\maketitle

\begin{abstract}
    We consider an online influence maximization problem in which a
  decision maker selects a node among a large number of possibilities
  and places a piece of information at the node. The node transmits
  the information to some others that are in the same connected
  component in a random graph. The goal of the decision maker is to reach
  as many nodes as possible, with the added complication that feedback is 
  only available about the degree of the selected node. Our main result 
  shows that such local observations can be sufficient for maximizing 
  global influence in two broadly studied families of random graph models: 
  stochastic block models and Chung--Lu models. With this insight, we propose
  a bandit algorithm that aims at maximizing local (and thus global) influence,
  and provide its theoretical analysis in both the subcritical 
  and supercritical regimes of both considered models. Notably, our performance 
  guarantees show no explicit dependence on the total number of nodes in the network,
  making our approach well-suited for large-scale applications.
\end{abstract}

\section{Introduction}
Finding most influential nodes in social networks has received increasing
attention in the last few years.
The problem has been cast in a variety of different
ways according to the notion of influence and the information
available to a decision maker. We refer the reader to
\cite{KeKlTa03,ChenSIM10,ChLaCa13,VaLaSc15,CaVa16,WeKvVaVa17,WaCh17} and the references therein 
for recent progress in various directions. In the present paper, we consider the problem of maximizing influence in a sequential setup 
where the learner has only partial information about the set of influenced nodes.

Specifically, we define and explore a
sequential decision-making model in which the goal of 
a decision maker is to find a node among a set $V$ of $n$ possible nodes
with maximal (expected) influence.
In our model, at every time instance $t=1,\ldots,T$, the $n$ nodes
form the vertex set of a random graph $G_t$ such that node $i$ and
node $j$ are connected in $G_t$ by an (undirected) edge with
probability $p_{i,j}$. All edges are present independently of each
other and the graphs $G_1,\ldots,G_T$ are independent and identically
distributed. If the decision maker selects a node $v_t\in V$ at time
$t$, then the information placed at the node spreads to the entire
connected component of $v_t$ in the graph $G_t$. The goal of the
decision maker is to spread the information as much as possible, which can be formulated as maximizing a notion of \emph{rewards} 
corresponding to the number of vertices in the connected component containing the selected node. 

In this paper, we study a setting where the decision maker has no prior knowledge of the distribution of $G_t$, so 
it has to learn about this distribution on the fly, while simultaneously attempting to maximize the total rewards. This gives rise to a 
dilemma of \emph{exploration versus exploitation}, which is commonly studied within the framework of \emph{multi-armed bandits} 
\citep{bubeck12survey}. Indeed, if the decision maker could observe the set of all influenced nodes in every round, the sequential influence 
maximization problem outlined above could be naturally formulated as a \emph{stochastic multi-armed bandit} problem 
\citep{LR85,auer2002finite}. However, this direct approach has multiple setbacks. First off, in most practical applications, the number $n$ 
of nodes is so large that one cannot even hope to maintain individual statistics about each of them, let alone expect any algorithm to 
identify the most influential node in any reasonable time. More importantly, in most cases of practical interest, tracking down the set of 
\emph{all} influenced users may be difficult or downright impossible due to privacy and computational considerations. This motivates the 
study of a more restrictive setting where the decision maker has to manage with only partial observations of the set of influenced nodes.

Formally, we address this latter challenge by considering a more realistic observation model, where after selecting a node $v_t$  
to be influenced, the learner only observes the number of \emph{immediate neighbors} of $v_t$ in the realized random graph $G_t$ (i.e., the 
degree of $v_t$ in $G_t$). This model brings up the following important question: is it possible to maximize \emph{global} influence while 
only having access to such local measurements? Our key technical result is answering this question in the positive for some broadly studied 
random graph models. Specifically, we show that, assuming that the graphs $G_t$ are generated from certain stochastic block models 
\citep{Abb17} or a Chung--Lu model \citep{CL02}, maximizing local influence is equivalent to maximizing global influence.

This observation motivates our algorithmic approach that applies ideas from the multi-armed bandit literature to try and maximize the local 
influence of each selected node. In order to analyze the performance of our algorithms, we adapt the standard notion of \emph{regret} from 
said literature to fit our needs. The traditional notion of regret measures the difference between the cumulative reward of choosing a 
maximally influential node in each round and the cumulative reward the decision maker achieves during the $T$ rounds of the game. This 
definition, however, is rather problematic in our problem setup: as mentioned above, the number $n$ of all nodes is typically so large that 
finding a maximally influential node is computationally infeasible, thus making the task of competing with this benchmark unreasonably 
complicated. To resolve this issue, we consider the notion of \emph{quantile regret} that compares the learner's performance to the top 
$\alpha$-fraction of all nodes (\citep{CFH09,CV10,LS14,KvE15}). Our main result is showing both instance-dependent bounds of order 
$\mathop{\mbox{log}} T$ and worst-case bounds of order $\sqrt{T}$ on the quantile regret of our algorithm. Notably, our bounds hold for 
both the subcritical and supercritical regimes of the random-graph models considered, and show no explicit dependence on the number of 
nodes $n$.

Related online influence maximization algorithms consider more general classes of networks, but make more restrictive assumptions about the 
interplay between rewards and feedback. One line of work explored by \citet{WeKvVaVa17,WaCh17} assumes that the algorithm receives 
\emph{full feedback} on where the information reached in the previous trials (i.e., not only the number of influenced nodes, but their 
exact identities and influence paths, too). Clearly, such detailed measurements are nearly 
impossible to obtain in practice, as opposed to the local observations that our algorithm requires.
Another related setup was considered by \citet{CaVa16}, whose algorithm only receives feedback about the nodes that were 
directly influenced by the chosen node, but the model does not assume that neighbors in the graph share the information to further 
neighbors and counts the reward only by the nodes directly connected to the selected one. That is, in contrast to our work, this work does 
not attempt to show any relation between local and global influence maximization.
One downside to all the above works is that they all provide rather conservative performance guarantees:
On one hand, \citet{WeKvVaVa17} and~\citet{CaVa16}  are concerned with worst-case regret bounds that 
uniformly hold for all problem instances for a fixed time horizon $T$. On the other hand, the bounds of \citet{WaCh17} depend on 
topological (rather than probabilistic) characteristics of the underlying graph structure, which inevitably leads to conservative 
results. For example, their bounds instantiated in our graph model lead to a regret bound of order $n^3 \log T$, which is virtually void of 
meaning in our regime of interest where $n$ is very large (e.g, in the order of billions). In contrast, our bounds do not show explicit 
dependence on $n$. In this light, our work can be seen as the first one that takes advantage of a specific graph structure to 
obtain strong instance-dependent global performance guarantees, all while having access to only local observations.

The rest of the paper is organized as follows. In Section~\ref{sec:prelim} we formally
introduce the regret minimization problem and the notation. In Section~\ref{sec:local}, we introduce our key technical results that show 
the connection between local and global influence maximization. We describe our algorithm and state its performance guarantees in 
Section~\ref{sec:alg}. We describe the main structure of the analysis in Section~\ref{sec:analysis} and discuss our results in 
Section~\ref{sec:conc}.

\section{Preliminaries}
\label{sec:prelim}
We now describe our problem and model assumptions more formally. We consider the problem of sequential influence maximization on a 
fixed finite graph $\pa{V,E}$, formalized as a repeated interaction scheme between a learner and its environment. In this setup, the following 
steps are repeated for each round $t=1,2,\dots$:
\begin{enumerate}
	\item the learner picks a vertex $A_t \in V$,
	\item the environment generates a subgraph $G_t$ of $(V,E)$,
	\item the learner observes the degree of the node $A_t$, denoted as $Y_{A_t, t}$,
	\item the learner earns the reward $r_{A_t,t} = |C_{A_t,t}|$, where the 
	set $C_{a,t}$ is the connected component associated with vertex $a$:
	\[
	C_{a,t} = \ev{v\in V: \mbox{$v$ is connected to $a$ by a path in $G_t$}}~.
	\]
\end{enumerate}
We stress that the learner does \emph{not} observe the reward, only the number of its immediate neighbors.  
Define $c_a$ as the expected size of the connected component associated with the node $a$:  $c_a = \E|C_{a,t}|$. Ideally, we would be 
interested in designing algorithms that minimize the \emph{expected regret} defined as
\begin{equation}\label{eq:regret}
R_T = \max_{a\in V} \sum_{t=1}^T \pa{c_a - c_{A_t}}.
\end{equation}
We would ideally aspire to design algorithms guaranteeing that the regret grows sublinearly in $T$. However, as we are 
interested in settings where the total number of nodes $n$ is very large, this goal can be seen as far too ambitious: even with a fully 
known random graph model, finding the optimal node maximizing $c_a$ is computationally infeasible. Such computational issues have lead to 
alternative definitions of the regret such as the \emph{approximation regret} \cite{KKL09,CWY13,SG09} or the \emph{quantile 
	regret} \citep{CFH09,CV10,LS14,KvE15}. 

In the present paper, we consider the $\alpha$-quantile regret as our performance measure, which, instead of measuring the learner's 
performance against the single best decision, uses a near-optimal action as a baseline. For a more technical definition, let $a_1, a_2, 
\dots, a_n$ be an ordering of the nodes satisfying  $c_{a_1} \le c_{a_2} \le \dots \le c_{a_n}$, and  denote the $\alpha$-quantile over 
the mean rewards as $c^*_\alpha = c_{a_{\lceil(1-\alpha) n\rceil}}$. Then, also defining the set 
$V^*_{\alpha} = \{a_{\lceil(1-\alpha) n\rceil}, \dots, a_{n}\}$ as the set of $\alpha$-near-optimal nodes, we define the 
$\alpha$-quantile regret as
\begin{equation}\label{eq:qregret}
R^{\alpha}_T = \min_{a\in V^*_{\alpha}} \sum_{t=1}^T \pa{c_a - c_{A_t}} = \sum_{t=1}^T \pa{c^*_{\alpha} - c_{A_t}}. 
\end{equation}

We will make the assumption that each $G_t$ is drawn from a fixed (and unknown) distribution of \emph{inhomogeneous random 
	graphs} (IRG, see, e.g.,\cite{Bollobas:2007:PTI:1276871.1276872}). In this model, we assume that $(V,E)$ is the complete graph over $n$ 
vertices and each edge $(i,j)$ is present with probability $p_{ij} (=p_{ji})$, independently of all other edges. The inhomogeneous random 
graph can be parametrized by the symmetric positive matrix $\overline{A}$, such that the probability of $i$ and $j$ being connected is 
given by $p_{ij} = \overline{A}_{ij}/n$. We will assume that each element $\overline{A}_{ij}$ of the matrix is $O(1)$ as $n$ grows large. 
This 
assumption corresponds to assuming that the graphs $G_t$ are \emph{sparse}, meaning that the expected degree of each vertex remains bounded 
as $n$ grows. This assumption makes the problem both more realistic and challenging: 
denser graphs are connected with high probability, making  the problem
essentially vacuous.
We will also use the notation $A = \overline{A}/n$.  The random graph from the above distribution 
is denoted as \irg.

We consider two fundamentally different regimes of the parameters \irg: the \emph{subcritical} case 
in which the size of the largest connected component is sublinear in $n$ (with high probability), and the \emph{supercritical} case where the largest connected 
component is at least of size $c n$ for some $c>0$ with high probability.
(We say that an event holds \emph{with high probability} if its probability converges to one as $n\to \infty$.)
Such a
connected component of linear size is called a \emph{giant component}.
These regimes can be formally characterized with the help of 
the matrix $A$.
Letting $\lambda$ be the the largest eigenvalue of $A$, we call \irg  subcritical if $\lambda < 1$ and supercritical if 
$\lambda > 1$.
It follows from \cite[Theorem 3.1]{Bollobas:2007:PTI:1276871.1276872} that, with high probability,
\irg has a giant component if it is supercritical, while the number of vertices in the
largest component is $o(n)$ with high probability if it is subcritical.

Within the class of inhomogeneous random graphs, we will focus on two families of random graphs: \emph{stochastic block models} and 
\emph{Chung--Lu models}, as discussed below.

\subsection{Stochastic block models}\label{sec:sbm}
First, we make the assumption that each $G_t$ is drawn from a \emph{stochastic block model} (SBM). 
In this random-graph model, the probabilities $p_{ij}$ are defined through the notion of 
\emph{communities}, defined as elements of a partition $H_1,\dots,H_S$ of the set of vertices $V$. We will refer to the index $m$ of 
community $H_m$ as the \emph{type} of vertices belonging to $H_m$. Each community $H_m$ contains $\alpha_m n$ nodes (assuming 
without loss of generality that $\alpha_m n$ is an integer). With the help of the community structure, the probabilities $p_{ij}$ 
are constructed as follows: if $i\in H_\ell$ and $j\in H_m$, the probability of $i$ and $j$ being connected is given by $p_{ij} = 
\frac{K_{\ell,m}}{n}$, where $K$ is a symmetric matrix of size $S\times S$, with positive elements. The random graph from the above 
distribution is denoted as \sbm.

In an SBM, identifying a node with maximal reward amounts to finding a node 
from the most influential community. Consequently, it is easy to see that choosing $\alpha$ such that $\alpha > 
\min_m \alpha_m$, the near-optimal set $V^*_{\alpha}$ will exactly correspond to the set of optimal nodes, and thus the quantile 
regret \eqref{eq:qregret} will coincide with the regret \eqref{eq:regret}.

Throughout the paper, we will consider SBM's satisfying the following assumption:
\begin{assumption}\label{ass:1}
	$K_{i,j} = k > 0$ for all $i\neq j$.
\end{assumption}
Intuitively, this assumption requires that nodes $i,j$ belonging to different communities are connected with the same probability, 
regardless of the exact identity of $m(i)$ or $m(j)$. Additionally, our analysis in the supercritical case will make the following 
natural assumption:
\begin{assumption}\label{ass:2}
	For all $i$, $ K_{i, i} \ge k$.
\end{assumption}
In plain words, this assumption requires that the density of edges within communities is larger than the density of 
edges between communities.

\subsection{Chung--Lu models}\label{sec:chunglu}
We will also consider another natural IRG model that is closely related to many random graph models. This is the 
so-called \emph{Chung--Lu model} (sometimes referred to as \emph{rank-1 model}) as first defined by \citet{CL02} (see also 
\citep{chung2006complex,Bollobas:2007:PTI:1276871.1276872}), where the edge probabilities are defined through the positive 
vector $w\in\real^n$, with elements of the matrix given by $\overline{A}_{ij} = w_i w_j$. In other words, the Chung--Lu model 
considers rank-1 matrices of the form $\overline{A} = ww\transpose$.

The random graph from the induced IRG distribution is denoted as \chlu. Such models can be shown to exhibit several interesting 
properties. For instance, if $w$ is a sequence satisfying a power law, then \chlu is a power law model, which allows one to model 
various real world networks including social networks \citep{chung2006complex}.

\section{From local to global influence maximization}\label{sec:local}
Having described the setting, we can finally ask the question: is it possible to maximize global influence using only local observations? 
Our main technical results show that, perhaps surprisingly, the most influential nodes are actually identifiable from such feedback in the 
models discussed in Sections~\ref{sec:sbm} and~\ref{sec:chunglu}.

To be specific, we recall that $Y_{a,t}$ stands for the degree of node $a$ in the realized graph $G_t$, and define $\mu_a = \E Y_{a,1}$ as 
the expected degree of node $a$. We also define $c^* = \max_a c_a$ and $\mu^* = \max_a \mu_a$. Our main technical result is proving that 
nodes with the largest expected degrees $\mu^*$ are exactly 
the ones with the largest influence $c^*$, in both the SBM and the Chung--Lu models, across both the subcritical and supercritical regimes. 
We formally state these results below.
\begin{proposition}\label{prop:subcr}
	Suppose that
	\begin{enumerate}
		\item $G$ is generated from a subcritical \sbm satisfying Assumption~\ref{ass:1}, or
		\item $G$ is generated from a subcritical \chlu. 
	\end{enumerate}
	Then, for any $a$ satisfying $\mu_a < \mu^*$, we have $c^* - c_a \le 2 c^*\pa{\mu^* - \mu_a} + O(1/n)$. In particular, for $n$  
	sufficiently large, we have $\argmax_{a} c_a = \argmax_{a} \mu_a$.
\end{proposition}
\begin{proposition}\label{prop:supercrit}
	Suppose that
	\begin{enumerate}
		\item $G$ is generated from a supercritical \sbm satisfying Assumptions~\ref{ass:1} and~\ref{ass:2}, or
		\item $G$ is generated from a supercritical \chlu.
	\end{enumerate}
	Then, for any $a$ satisfying $\mu_a < \mu^*$, we have $c^* - c_a \le c^*\pa{\mu^* - \mu_a} + o(n)$. In particular, for $n$  
	sufficiently large, we have $\argmax_{a} c_a = \argmax_{a} \mu_a$.
	
\end{proposition}
These propositions are proved in Appendix~\ref{sec:subcrit} and~\ref{sec:supercrit}, respectively. To the best of our knowledge, these 
results are novel and can be of independent interest. The proofs rely on the concept of 
multi-type Galton--Watson branching processes, which are briefly introduced alongside some of their main properties in 
Appendix~\ref{sec:branching}.

\section{Algorithm and main results}
\label{sec:alg}

\begin{algorithm}[t]
	\caption{\ducbv}
	\label{alg:thealgexp}
	\textbf{Parameters:} A set of nodes $V_0 \subseteq V.$\\
	\textbf{Initialization:} Select each node in $V_0$ once. Observe the degree $X_{a,a}$
	of vertex $a$ in the graph $G_a$ for $a=1,\ldots,|V_0|$.
	For each $a\in V_0$, set $N_a(|V_0|)=1$ and ${\hmu}_a(|V_0|)=X_{a,a}$.
	\\
	\textbf{For} $t = |V_0|, \dots T$, \textbf{repeat}
	\begin{enumerate}
		\item For each node, compute 
		\[
		U_a(t) = \sup \left\{ \mu: \mu - {\hmu}_a(t) + {\hmu}_a(t) \log \left(\frac{{\hmu}_a(t)}{\mu} \right)\le 
		\frac{3\log(t)}{N_a(t)} 
		\right\}.
		\]
		\item Select any node $A_{t+1} \in \arg\max_{a} U_a(t)$. 
		\item Observe degree $Y_{A_{t+1},t+1}$ of node $A_{t+1}$ in $G_{t+1}$
		and update
		\[
		{\hmu}_{A_{t+1}}(t+1)=\frac{N_{A_{t+1}}(t){\hmu}_{A_{t+1}}(t+1)+Y_{A_{t+1}, t+1}}{N_{A_{t+1}}(t)+1}~.
		\]
		Update $N_{A_{t+1}}(t+1)=N_{A_{t+1}}(t)+1$.
	\end{enumerate}
\end{algorithm}
We now present our learning algorithm, and provide its performance guarantees for the two regimes. Inspired by the observation that in the 
models that we consider, it is sufficient to identify nodes with maximal degree in order to maximize influence, we design a bandit 
algorithm that attempts to maximize the degrees of the influenced nodes. We propose to achieve this goal by employing a variant
of the kl-UCB algorithm proposed and analyzed by \cite{GaCa11,maillard11dmed,cappe:hal-00738209,Lai87}. More
precisely, we propose to use the observed degrees as rewards, and feed them to an instance of kl-UCB originally designed for 
Poisson-distributed rewards. A key technical challenge arising in the analysis is that the degree distributions do not actually belong to 
the Poisson family for finite $n$. We overcome this difficulty by showing that the degree distributions have a moment generating function 
bounded by those of Poisson distributions, and that this fact is sufficient for most of the kl-UCB analysis to carry through without 
changes.

A minor challenge is that, since we are interested in very large values of $n$, it is computationally infeasible to use \emph{all} 
nodes as separate actions in our bandit algorithm. To address this challenge, we propose to \emph{subsample} a set of representative nodes 
for kl-UCB to play on. The size of the subsampled nodes depends on the quantile $\alpha$ targeted in the regret 
definition~\eqref{eq:qregret} and the time horizon $T$. For clarity of presentation, we first propose a simple algorithm that assumes prior 
knowledge of $T$, and then move on to construct a more involved variant that adds new actions on the fly.

We first present our \klucb variant for a fixed set of nodes $V_0$ as Algorithm~\ref{alg:thealgexp}. We refer to this algorithm as 
$\mbox{\ducb}(V_0)$ (short for ``degree-UCB on $V_0$''). Our two algorithms mentioned above use \ducb$(V_0)$ as a subroutine: they are both 
based on uniformly sampling a large enough set $V_0$ of nodes so that the subsample includes at least one node from the top 
$\alpha$-quantile.

To simplify the presentation of our main results, let us introduce some more notation. Analogously to the $\alpha$-optimal reward 
$c^*_\alpha$, we define the $\alpha$-optimal degree $\mu^*_\alpha = \min_{a\in V^*_{\alpha}} \mu_a$, and the corresponding 
gap parameters $\Delta_{\alpha,i} = \pa{c_i - c^*_\alpha}_+$ and $\delta_{\alpha,i} = \pa{\mu_i - \mu^*_\alpha}_+$. Finally, define 
$\Delta_{\alpha,\max} = \max_i \Delta_{\alpha,i}$.
We first present a performance guarantee of our simpler algorithm that assumes knowledge of $T$. This method uniformly samples 
a subset of size
\begin{equation}\label{eq:v0}
|V_0| =  \left \lceil \frac{\log T}{\log (1/(1- \alpha))} \right\rceil
\end{equation}
and plays \ducbv on the resulting set.
This algorithm satisfies the following performance guarantee:
\begin{theorem}\label{thm:reg_knowT}
	Let $V_0$ be a uniform subsample of $V$ with size given in Equation~\eqref{eq:v0} and define the event $\mathcal{E} = \ev{V_0 \cap 
		V^*_\alpha \neq \emptyset}$.
	If the number of vertices $n$ is sufficiently large, then the expected $\alpha$-quantile regret of \ducbv simultaneously satisfies
	\begin{align*}
	R^{\alpha}_T \le  \EEcc{ \sum_{i \in V_0}  \Delta_{\alpha,i}  \pa{\frac{\mu_\alpha^*\pa{18 + 27\log T}}{\delta_{\alpha,i}^2} + 3} 
	}{\mathcal{E}}
	+ \Delta_{\alpha,\max},
	\end{align*}
	where the expectation is taken over the random choice of $V_0$, and
	\begin{align*}
	R^{\alpha}_T \le  
	18c^* \sqrt{\frac{T \mu^* \pa{2 + 3\log T}^2}{\log(1/(1- \alpha))}} + \pa{\frac{3 \log T}{\log(1/(1-\alpha))} +4} \Delta_{\alpha,\max}.
	\end{align*}
\end{theorem}
\begin{algorithm}[t]
	\caption{\ducb-\textsc{double}$(\beta)$}\label{alg:doubling}
	\textbf{Parameters:} 
	$\beta \ge 2$.\\
	\textbf{Initialization: $V_0 = \emptyset$.}
	\\
	\textbf{For} $k = 1, 2 \dots $, \textbf{repeat}
	\begin{enumerate}
		\item Sample subset of nodes $U_k$ uniformly such that $|U_k| = \left\lceil\frac{\log \beta}{\log(1/(1- \alpha))}\right\rceil$.
		\item Update action set $V_{k} = V_{k-1} \cup U_k$.
		\item For rounds $t = \beta^{k-1}, \beta^{k-1} + 1, \dots, \beta^{k} -1$, run a new instance of \ducb$(V_k)$.
	\end{enumerate}
\end{algorithm}
For unknown values of $T$, we propose the \ducb-\textsc{double}$(\beta)$ algorithm (presented as Algorithm~\ref{alg:doubling}) that uses a 
doubling trick to estimate $T$. The following theorem gives a performance guarantee for this algorithm: 
\begin{theorem}\label{thm:doubling}
	Fix $T$, let $k_{\max}$ be the value of $k$ on which \ducb-\textsc{double}$(\beta)$ terminates, and define the event $\mathcal{E} = 
	\ev{V_{k_{\max}} \cap V^*_\alpha = \emptyset}$.
	If the number of vertices $n$ is sufficiently large, then the $\alpha$-quantile regret of
	\ducb-\textsc{double}$(\beta)$ simultaneously satisfies
	
	\begin{align*}
	R^{\alpha}_T  \le  \EEcc{\sum_{i \in V_{k_{\max}}} \Delta_i \pa{ \pa{\frac{18\mu^*}{\delta_{\alpha,i}^2} + 3} (\log_\beta T + 1) 
			+  \frac{27 \log \beta (\log_\beta T+1)^2 }{ 2\delta_{\alpha,i}^2}}}{\mathcal{E}} + \Delta_{\alpha, \max} \log_\beta T ,
	\end{align*}
	where the expectation is taken over the random choice of the sets $V_1,V_2,\dots$, and
	\begin{align*}
	R^{\alpha}_T \le 36c^* \sqrt{\frac{T  \pa{\mu^* + \log \pa{\beta T}}\log^2T }{\log(1/(1- \alpha))}} 
	+ \pa{\frac{3 \log^2 T}{\log(1/(1-\alpha))} + 4} \Delta_{\alpha,\max}.
	\end{align*}
\end{theorem}
We discuss the key features of the above regret bounds in Section~\ref{sec:conc}.

\section{Analysis}
\label{sec:analysis}

This section outlines the main ideas of the proofs of our main results, Theorems~\ref{thm:reg_knowT} and~\ref{thm:doubling}. Having 
established that, in order to minimize regret in our setting, it is sufficient to design an algorithm that quickly identifies the 
nodes with the highest degree, it remains to show that our algorithms indeed achieve this goal. We do this below by providing a bound on the 
expected number of times $\E{N_{a}(T)} = \E[{\sum_{t=1}^T \II{A_t = a}}]$ that our algorithm picks suboptimal node $a$ such that $c_a \le 
c^*$, and then using this guarantee to bound the regret.

Without loss of generality, we assume that $V_0 = \ev{1,2,\dots,|V_0|}$.  The key to our regret bounds is the following guarantee on the 
number of suboptimal actions taken by \ducbv. 
\begin{theorem}[Number of suboptimal node plays in \ducb] \label{thm:n_klucb}
	Define $\eta_i = \pa{\max_{j\in V_0} \mu_j- \mu_i} / 3$. 
	The number of times that any  node $i \in \ev{a: \mu_{a} < \max_{j\in V_0} \mu_j}$ is chosen by \ducbv
	satisfies
	\begin{equation}
	\mathbb{E} N_{i}(T) \le  \frac{\mu^* \pa{2 + 6 \log T}}{\eta_{i}^2} + 3~.
	\end{equation}
\end{theorem}
The proof is largely based on the analysis of the kl-UCB algorithm due to \citet{cappe:hal-00738209}, with some additional tools borrowed 
from \citet{2017arXiv170207211M}, crucially using that the degree distribution of each node is stochastically dominated by an appropriately 
chosen Poisson distribution. Specifically, letting $Z_i$ be a Poisson random variable with mean $\E{Y_{i,t}}$, we have $\E{e^{sY_{i,t}}} 
\le \E{e^{sZ_i}}$ for all $s$. Turns out that this property is sufficient for the \klucb analysis to go through in our case, which is an 
observation that may be of independent interest.

Due to space constraints, the proof of Theorem~\ref{thm:n_klucb} is deferred to Appendix~\ref{app:klucb}. The remainder of the section 
uses 
Theorem~\ref{thm:n_klucb} to prove our first main result, Theorem~\ref{thm:reg_knowT}. The proof of Theorem~\ref{thm:doubling} follows from 
similar ideas and some additional technical arguments, and is presented in Appendix~\ref{app:doubling}.
\begin{proof}[Proof of Theorem~\ref{thm:reg_knowT}]
	We first note that, with high probability, the size of $V_0$ guarantees that the subset holds at least one node from the set 
	$V^*_{\alpha}$: $\PP{\mathcal{E}} \ge 1- 1/T$.
	Then, the regret can be bounded as
	\begin{align*}
	\EE{R^{\alpha}_T} 
	&\le \PP{\mathcal{E}^c} T \Delta_{\alpha,\max} + 
	\EEcc{\sum_{t=1}^{T} \sum_{i \in V_0} \mathbb{I}[A_t = i]  \Delta_{\alpha,i}}{\mathcal{E}} \PP{\mathcal{E}}
	\\
	&\le 
	\Delta_{\alpha,\max} + \EEcc{ \sum_{i\in V_0}  \Delta_{\alpha,i}   \EE{N_{i}(T)} }{\mathcal{E}}~.
	\end{align*}
	Now, observing that $\delta_{\alpha,i} \le 3\eta_i$ holds under event $\mathcal{E}$, we appeal to Theorem \ref{thm:n_klucb} to obtain
	\begin{equation}
	R^{\alpha}_T \le   \Delta_{\alpha,\max} +  \EEcc{ \sum_{i\in V_0}  \Delta_{\alpha,i}  \pa{
			\frac{\mu^*\pa{18 + 27\log T}}{\delta_{i,\alpha}^2} + 3} }{\mathcal{E}},
	\end{equation}
	thus proving the first statement.
	
	Next, we turn to proving the second statement regarding worst-case guarantees. To do this, we appeal to 
	Propositions~\ref{prop:subcr} and~\ref{prop:supercrit} that respectively show $\Delta_i \le 2 c^* \delta_i + O(1/n)$ and $\Delta_i \le c^* 
	\delta_i + o(n)$ for the sub- and supercritical settings, and we use our assumption that $n$ is 
	large enough so that we have $\Delta_i \le 3 c^* \delta_i$ in both settings. 
	Specifically, we observe that $\delta_i = \Theta_n(1)$ by our sparsity assumption and $c^*$ is $\Theta_n(1)$ in the subcritical and 
	$\Theta_n(n)$ supercritical settings, so, for large enough $n$, the superfluous $O(1/n)$ and $o(n)$ terms can be 
	respectively bounded by $c^* \delta_i$. 
	To proceed, let us fix an arbitrary $\varepsilon > 0$ and
	split the set $V_0$ into two subsets: $U(\varepsilon) = \ev{a\in V_0: \delta_{\alpha,i} \le \varepsilon}$ and $W(\varepsilon) = V_0 
	\setminus U(\varepsilon)$. Then, under event $\mathcal{E}$, we have
	\begin{align*}
	\sum_{i\in V_0}  \Delta_{\alpha,i}   \EE{N_{i}(T)} &= \sum_{i\in U(\varepsilon)}  \Delta_{\alpha,i}   \EE{N_{i}(T)} +
	\sum_{i\in W(\varepsilon)}  \Delta_{\alpha,i}   \EE{N_{i}(T)}
	\\
	&\le 3c^* \varepsilon \sum_{i\in U(\varepsilon)} \EE{N_{i}(T)} + 3c^* \sum_{i\in W(\varepsilon)}  \delta_{\alpha,i} \pa{
		\frac{\mu^*\pa{18 + 27\log T}}{\delta_{\alpha,i}^2}} + 3 |W(\varepsilon)|\Delta_{\alpha,\max}
	\\
	&\le 3c^* \varepsilon T + 3c^* \sum_{i\in W(\varepsilon)}  \frac{\mu^*\pa{18 + 27\log T}}{\delta_{\alpha,i}} + 3 |V_0| \Delta_{\alpha,\max}
	\\
	&\le 3c^* \pa{\varepsilon T + |V_0| \frac{\mu^*\pa{18 + 27\log T}}{\varepsilon}} + 3 |V_0| \Delta_{\alpha,\max}
	\\
	&\le 6c^* \sqrt{T |V_0| \mu^*\pa{18 + 27\log T}} + 3 |V_0| \Delta_{\alpha,\max},
	\end{align*}
	where the last step uses the choice $\varepsilon = \sqrt{|V_0| \mu^*\pa{18 + 27\log T} / T}$. Plugging in the choice of $|V_0|$ concludes 
	the proof.
\end{proof}

\section{Discussion}\label{sec:conc}
Here we highlight some features of our 
results and discuss directions for future work.

\paragraph{Instance-dependent and worst-case regret bounds.}
Both of our main theorems establish two types of regret bounds. The first set of these bounds are polylogarithmic\footnote{Upon first 
	glance, the bound of Theorem~\ref{thm:reg_knowT} may appear to be logarithmic, however, notice that the sum involved in the bound has $\log 
	T$ elements, thus technically resulting in a bound of order $\log^2 T$.} in the time horizon $T$, but show strong dependence on the 
parameters of the distribution of the graphs $G_t$. Such bounds are usually called \emph{instance-dependent}, and they are 
typically interesting in the regime where $T$ grows large. However, these bounds become vacuous for finite $T$ as the gap parameters 
$\delta_{\alpha,i}$ approach zero. This issue is addressed by our second set of guarantees, which offers a bound of 
$\wt{O}\bpa{c^*\sqrt{|U|\mu^* T}}$ for some set $U\subseteq V$ that holds simultaneously for all problem instances without becoming 
vacuous in any 
regime. Such bounds are commonly called \emph{worst-case}, and they are typically more valuable when optimizing performance over a fixed 
horizon $T$.

\paragraph{Dependence on graph parameters.}
A notable feature of all our bounds is that they show no explicit dependence on the number of nodes $n$. This is enabled by our notion of 
$\alpha$-quantile regret, which allows us to work with a small subset of the total nodes as our action set. 
Instead of $n$, our bounds depend on the size of some suitably chosen set of nodes $U$, which is of the order $\mathop{\mbox{polylog}}T / 
\log(1/(1-\alpha))$. Notice that this gives rise to an interesting tradeoff: choosing smaller values of $\alpha$ inflates the regret 
bounds, but, in exchange, makes the baseline of the regret definition stronger (thus strengthening the regret notion itself). While the 
exact tradeoff seems very complicated to quantify in general, it is clear that setting $\alpha$ as the proportion of the smallest community 
in SBMs strengthens the regret baseline as much as possible. 

Of course, having no \emph{explicit} dependence on $n$ does not mean that our bounds are completely independent of the size of the 
graph. In fact, it is natural to expect that the regret scales with the general magnitude of the rewards. Our bounds precisely achieve 
such a natural dependence: all our bounds scale linearly with the maximal expected reward $c^*$, which is of $\Theta_n(1)$ in the 
subcritical case, but is $\Theta_n(n)$ in the supercritical case.

\paragraph{Tightness of our bounds.}
In terms of dependence on $T$, both our instance-dependent and worst-case bounds are near-optimal in their respective 
settings: even in the simpler stochastic multi-armed bandit problem, the best possible regret bounds are $\Omega_T(\log T)$ and 
$\Omega_T(\sqrt{T})$ in the respective settings \citep{auer2002finite,auer2002bandit,bubeck12survey}. The optimality of our bounds with 
respect to other parameters such as $c^*$, $\mu^*$ and $n$ is less clear, but we believe that these factors cannot be improved 
substantially for the models that we studied in this paper. As for the subproblem of identifying nodes with the highest degrees, we believe 
that our bounds on the number of suboptimal draws is essentially tight, closely matching the classic lower bounds by \citet{LR85}.

\paragraph{Our assumptions.}
One may wonder how far our argument connecting local and global influence maximization can be stretched. Clearly, not every random graph 
model enables establishing such a strong connection. In fact, even within the class of stochastic block models, one can construct an 
instance (not satisfying Assumption~\ref{ass:1}) that does not have the property we desire.
It is a challenging problem to characterize the class of inhomogeneous random graphs in which
maximizing local and global influences are equivalent. Nevertheless, we believe that our techniques can be generalized to 
maximize global influence with more informative local feedback structures (e.g., working with observations from a slightly broader 
neighborhood of the chosen nodes).

Finally, let us comment on our condition that the number of vertices $n$ needs to be  ``sufficiently large''. We regard this condition as a 
technical artifact due to our proofs relying on asymptotic analysis.
We expect that the required monotonicity property holds for small values of $n$ under mild conditions. Whenever this is the case, the regret 
bounds of Theorems \ref{thm:reg_knowT} and \ref{thm:doubling} remain valid.

\bibliographystyle{abbrvnat}

\newpage
\appendix

\section{Multi-type branching processes}
\label{sec:branching}
One of the most important technical tools for analyzing the component structure of
random graphs is the theory of \emph{branching processes}, see \cite{hofstad_2016}.
Indeed, while the connected components $C_{a}$ of an inhomogenous random graph \irg have a complicated 
structure, many of their key properties may be analyzed through the concept of
multi-type Galton--Watson processes.
Specifically, we use Poisson multi-type Galton--Watson branching processes with $n$ types, parametrized by an $n \times n$ matrix $A$ with strictly posive elements.
The branching process tracks the evolution of a set of \emph{individuals} of various types. Starting in round 
$n=0$ from a single individual of type $i$, each further generation in the Galton--Watson process $\GW_A(i)$ is generated by each 
individual of each type $i$  producing $X_{k,i}\sim \poi(A_{i,j})$ new individuals of each type $j$. Therefore, the size of the offspring of 
the individual of type $i$ is $\sum_{j=1}^n X_{i,j} \sim \poi( \sum_{j=1}^n A_{i,j})$. We also define the vector
$b \in \real^n$ with coordinates  $b_i = \EE{\sum_{j=1}^n X_{i,j}} = \sum_{j=1}^n A_{i,j}$, $i=1,\ldots,n$. 

Our analysis below makes use of the following quantities associated with the multi-type
branching process:
\begin{enumerate}
	\item $Z_n(i)$ is the number of individuals in generation $n$ of $\GW_A(i)$ (where $ Z_0(i) = 1$);
	\item $X(i)$ is the \emph{total progeny}, that is, the total number of individuals generated by $\GW_A(i)$ and its expectation is denoted by $x_i = \EE{X(i)}$;
	\item $\rho_i$ is the \emph{probability of survival}, that is, the probability that $X(i)$ is infinite.
\end{enumerate}
We finally define a non-linear operator $\Phi_A:\real^n \ra \real^n$ that plays a central role in our analysis: for a vector $f\in \real^n$, define 
each coordinate of $\Phi_A(f)$ as
\begin{equation}\label{eq:survival_proba}
\bpa{\Phi_A(f)}_j = 1 - e^{- \pa{A f}_j}, \quad j=1,\ldots,n~,
\end{equation}
where $\pa{A f}_j$ denotes the $j$-th coordinate of $Af$.
Abusing notation, we use the shorthand form $\Phi_A(f) = 1 - e^{- A f}$. 
Clearly, if $f$ has nonnegative components, then  $ \pa{\Phi_A(f)}_j \in [0,1]$ for all $j$.

Bollob\'as, Janson, and Riordan \cite{Bollobas:2007:PTI:1276871.1276872} establish a connection between the sizes of
connected components of IRG, the 
survival probability of a  branching process $\GW_A(i)$, and the norm of the matrix $A$. 

As shown in \cite{Bollobas:2007:PTI:1276871.1276872}, the operator $\Phi_A$
can be directly used for characterizing the probability 
$\rho_i$ of survival of the process $\GW_A(i)$ for all $i$.
By their Theorem 6.2, the vector $\rho = (\rho_1, \dots, \rho_n)$ is one of the 
solutions of the non-linear fixed-point equation $\Phi_A(f) = f$.
Furthermore, if the largest eigenvalue of the matrix $A$ satisfies $\lambda_{\max}(A) < 1$,
then $\rho_i =  0$ for all $i=1,\ldots,n$.
On the other hand, $\lambda_{max} > 1$ implies that the vector $\rho$ is the
\emph{maximal} fixed point of the operator $\Phi_A$
\cite[Lemma 5.8.]{Bollobas:2007:PTI:1276871.1276872} also implies that
when $\lambda_{max} > 1$, all components of $\rho$ are positive.

\section{The proof of Proposition~\ref{prop:subcr}}
\label{sec:subcrit}
The proof consists of the following steps:
\begin{itemize}
	\item proving that $c_i - c_j = x_i - x_j + O(1/n)$ (Lemmas~\ref{mean_dominate}, \ref{mean_difference}),
	\item proving that $x_i > x_j$ implies $b_i > b_j$ (Lemma~\ref{mean_order_sbm}, \ref{mean_order_chlu}),
	\item observing that $b_i = \mu_i + O(1/n)$.
\end{itemize}
These facts together lead to Proposition~\ref{prop:subcr}, given that $n$ is large enough to suppress the effects of the residual 
terms.
Before stating and proving the lemmas, we state some useful technical tools.
Since we suppose that \irg is subcritical, we have $\PP{X(i) = \infty} = 0$ and $x_i = \E{X(i)}$ is finite.
First observe that the vector $x$ of expected total progenies satisfies the system of linear equations 
\[
x =  e + A x~,
\] 
where $e$ is the vector with $e_i = 1$ for all $i$. Notice that, by its definition, the vector $b$  can be succinctly written as $b = A  e$.

Armed with this notation, we can analyze the relation between $b_i$ and $x_i$ in a straightforward way:
\begin{lemma}[Coordinate order for mean of the total progeny in the SBM]
	\label{mean_order_sbm}
	Assume that \sbm is subcritical and that $K_{m\ell} = k>0$ holds for all $m\neq \ell$. 
	If two coordinates of $b$ are such that $b_i > b_j$, then we have $x_i > x_j$, and $x_i - x_j \le 2 x^* \pa{b_i - b_j}$.
	
\end{lemma}
\begin{proof}
	For and SBM  with $S$ blocks, the system of equations $x = e + A x$ can be equivalently written as  $x' = e + M x'$, for $M = K \mbox{diag}(\alpha)\in\real^{S\times S}$, and $x'\in\real^S$, with $x'_m$ now standing for the expected total progeny associated with any 
	node of type $m$. Similarly, we define $b'_m$ as the expected degree of any node of type $m$.
	Notice that the system of equations $x' = e + M x'$ satisfied by $x'$ can be rewritten as $(I - M)x' = e$, where $I$ is the $S\times S$ 
	identity matrix. By exploiting our assumption on the matrix $K$ and defining $\gamma_m = K_{m,m} - k$,  this can be further rewritten as
	$$
	\left(\begin{pmatrix}
	1 - \alpha_{1}\gamma_1 & & \\
	& \ddots & \\
	& & 1 - \alpha_{S} \gamma_S
	\end{pmatrix} - 
	k  \begin{pmatrix}
	\alpha_{1} & \alpha_{2} & \cdots & \alpha_{S} \\
	\alpha_{1} & \alpha_{2} & \cdots & \alpha_{S} \\
	\vdots  & \vdots  & \ddots & \vdots  \\
	\alpha_{1} & \alpha_{2} & \cdots & \alpha_{S} 
	\end{pmatrix}\right) x'  = e,
	$$
	which means that for any $m$, $x_m'$ satisfies
	$$
	x_m' = \frac{1+k(\alpha\transpose x')}{1 - \alpha_m \gamma_m}.
	$$
	Also observe that $$b_m' = k(\alpha^T \bar{1}) + \alpha_m\gamma_m,$$
	
	so, for any pair of types $m$ and $\ell$, we have 
	$$
	x'_m - x'_\ell  = \frac{(1+k(\alpha\transpose x'))(\alpha_m\gamma_m - 
		\alpha_\ell\gamma_\ell)}{(1 - \alpha_m\gamma_m)(1 - \alpha_\ell\gamma_\ell)},
	$$
	which proves the first statement.
	
	To prove the second statement, observe that for any pair $\ell$ and $m$ of communities, we have either $\alpha_m \le \frac 12$ or 
	$\alpha_\ell \le \frac 12$ (otherwise we would have $\alpha_m + \alpha_\ell > 1$). To proceed, let $\ell$ and $m$ be such that $x'_m 
	\ge x'_\ell$, and let us study the case $\alpha_\ell \le \frac 12$ first. Here, we get
	\[
	\begin{split}
	x'_m - x'_\ell &\le \frac{(1+k(\alpha\transpose x'))(\alpha_m\gamma_m - \alpha_\ell\gamma_\ell)}{(1 - \alpha_m\gamma_m)(1 - 
		\alpha_\ell\gamma_\ell)} = \frac{(\alpha_m\gamma_m - \alpha_\ell\gamma_\ell)}{(1 - \alpha_\ell\gamma_\ell)} x'_m
	\\
	&\le \frac{(\alpha_m\gamma_m - \alpha_\ell\gamma_\ell)}{(1 - \gamma_\ell/2)} x'_m \le 2 x'_m (b_m' - b'_\ell).
	\end{split}
	\]
	In the other case where $\alpha_m \le \frac 12$, we can similarly obtain
	\[
	x_m' - x_\ell' \le 2 x'_\ell (b_m' - b'_\ell) \le 2 x'_m (b_m' - b'_\ell).
	\]
	This concludes the proof.
\end{proof}
\begin{lemma}[Coordinate order for mean of the total progeny in the \chunglu model]
	\label{mean_order_chlu}
	Assume that \chlu is subcritical. If two coordinates of $b$ are such that $b_i > b_j$, then we have $x_i > x_j$ and $x_i - 
	x_j \le x^* (b_i - b_j)$.
\end{lemma}
\begin{proof}
	For the system of equations $x = e + A x$ the coordinates $x_i$ have the form 
	$$
	x_i = 1 + \frac{1}{n}\cdot w_i\pa{\sum_{j = 1}^{n}w_j x_j},
	$$
	which implies that $w_i \ge w_j$ holds if and only if $x_i \ge x_j$. This observation implies for $x^* = \max_i x_i$
	$$x_i - x_j \le \frac{1}{n}\cdot (w_i - w_j)\pa{\sum_{j = 1}^{n}w_j} x^* = \pa{b_i - b_j} x^*,$$
	thus concluding the proof.
\end{proof}
The next two lemmas establish the relationship between the expected component size $c_i$ of vertex $i$ and the expected
total progeny $x_i$ of the multi-type branching process seeded at vertex $i$.
\begin{lemma}
	\label{mean_dominate}
	For any $i$, the mean of the connected component associated with type $i$ is bounded by the mean of the total progeny: $c_i \le x_i$. 
\end{lemma}
\begin{proof}
	The proof of the lemma uses the concept of \emph{stochastic dominance} between random variables.
	We say that
	the random variable $X$ is \emph{stochastically dominated} by the random variable $Y$ when, for every $x \in\mathbb{R}$, the inequality
	$\PP{X \le x} \ge \PP{Y \le x}$ holds. We denote this by $X  \preceq Y$.
	
	Now fix an arbitrary $i\in[n]$ and let $Y_{i,1}, Y_{i,2}, \dots, Y_{i,n}$ be independent Bernoulli random 
	variables with respective parameters $(\overline{A}_{i,1}/n, \overline{A}_{i,2}/n, \dots ,  \overline{A}_{i,i}/n, \dots, \overline{A}_{i,n}/n)$. 
	Consider a multitype binomial branching process where the individual of type $i$ produce  $Y_{i,j}$ individuals of type $j$, and 
	let $X_{\ber}(i)$ denote its total progeny when started from an individual of type $i$. 
	Recalling the Poisson branching process defined in Section~\ref{sec:branching} with offspring-distributions $X_{i,j}$, we can show 
	$X_{\ber}(i)\preceq X(i)$ using the relation $Y_{i,j} \preceq X_{i,j}$.

	Considering a node $a$ of type $i$, we can use Theorem 4.2 of \cite{hofstad_2016} to bound the size of the the connected component 
	$C_a$ as $|C_a| \preceq X_{\ber}(i)$, which implies by transitivity of $\preceq$ that $|C_{a_i}| \preceq X(i)$. The proof is concluded by 
	appealing to Theorem~2.15 of \cite{hofstad_2016} that shows that stochastic domination implies an ordering of the means.
\end{proof}
Next we upper bound the surplus that appears in the domination by the branching process:
\begin{lemma}
	\label{mean_difference}
	$x_i - c_i = O(\frac{1}{n})$~. 
\end{lemma} 

\begin{proof}
	Consider an exploration process in the realization $G_t$ of a random graph \irg starting from a node $a$ of type $i$.
	The process explores the nodes in a sequential way, by first visiting the neighboring nodes of the initial node, then
	moving on to the neighbors of the neighbors, and so on. The process stops after having explored the whole connected component $C_a$.
	
	Also the Bernoulli multitype branching process $W_{\ber}(i)$ with the set of parameters $\mathbb{B}_j$, for $j \in [n]$, where parameters 
	$\mathbb{B}$ correspond to the \irg. 
	Denote the tree naturally defined by the exploration process of the connected component by $\TT$, and also the tree defined by the analogously defined Poisson 
	exploration process of the branching process tree  by $\TT_{\poi}$. The proof relies on the fact that the total number of nodes visited by the exploration process 
	can be upper bounded by the total progeny of the corresponding branching process \citep[Section~4.1]{hofstad_2016}.
	
	In order to estimate the difference $|\TT| -|C_a|$, note that for each step
	of the exploration process, the number of nodes that have been already explored can be upper bounded by $|C_a|$ so we have $|C_a| \le|\TT| 
	$. 
	Let $\Sw$ be a set of nodes counted more than once.  We call $|\Sw|$ the \emph{surplus} whose expectation may be
	written as follows:
	\[
	\EE{|\Sw|} =\EE{ \sum_{v\in V} \mathbb{I} \{v \in \Sw \} } 
	= \sum_{k=1}^{\infty} \PP{|\TT|  = k }   \sum_{v \in C_a} \EEcc{\mathbb{I}\{v \in \Sw\}}{ |\TT| = k }.
	\]
	Define $\overline{A}_{\max}=\max_{i,j} \overline{A}_{i,j}$ be the maximal element of the matrix $\overline{A}$. 
	
	Then, by the union bound, the probability of an arbitrary node $a'$ is counted more than once can be upper bounded as
	$$\PP{a' \in \Sw} \le \frac{\overline{A}_{\max} |C_a|}{ n}~,$$
	and we also have
	$$ \EEcc{\mathbb{I}\{a' \in \Sw\}}{ |\TT| = k } \le \frac{\overline{A}_{\max} k}{ n}~. $$
	Since $|C_a|\le |\TT|$, we may upper bound the sum as
	$$ \sum_{v \in C_a} \EEcc{\mathbb{I}\{v \in \Sw\}}{|\TT| = k} \le  \frac{\overline{A}_{\max} k^2}{n}~. $$
	Using our expression for $\E|\Sw|$, we get
	\begin{align*}
	\mathbb{E} | \Sw| \le \sum_{k=1}^{\infty} \PP{|\TT|  = k }  \frac{\overline{A}_{\max} k^2}{n} = \frac{\overline{A}_{\max} \mathbb{E} 
		|\TT|^2}{ n}~.
	\end{align*}
	Now we notice that, by the Le Cam's theorem, the  total variation distance between the sum of Bernoulli distributed random variables with parameters $(\overline{A}_{i, 1}/n, \dots, \overline{A}_{i, n/n})$ and  the Poisson distribution $\poi(\sum_{j=1}^{n} \overline{A}_{i,j}/n)$ is at most $2(\sum_{j=1}^{n} \overline{A}^2_{i,j})/n$. 
	Using this fact and that the moments of the total progeny do not scale with $n$ (cf.~Theorem~1 of \citealp{Hua12}), we obtain the result as
	\begin{align*}
	\mathbb{E} | \Sw| \le \frac{\overline{A}_{\max} \E |\TT|^2}{n} \le  \frac{\overline{A}_{\max} \mathbb{E} |\TT_{\poi}|^2 }{ n} + 	O\left(\frac{1}{n}\right)= 
	O\left(\frac{1}{n}\right)~.
	\end{align*}
\end{proof}

\section{The proof of Proposition~\ref{prop:supercrit}}
\label{sec:supercrit}
The proof relies on some known properties of the largest connected
component in \irg in the supercritical regime. 
We denote the largest and second-largest connected components of $G_t$ by
$C_1(G_t)$ and $C_2(G_t)$, respectively. Recall that the survival probability of the branching process $W_A(i)$ is denoted as $\rho_i$.
The following properties are proved by \cite{Bollobas:2007:PTI:1276871.1276872}:
\begin{itemize}
	\item If \irg is supercritical, then, with high probability, $C_1 = \Theta(n)$;
	\item $C_1(G_n)/n \to \sum_{i \in S} \alpha_i \rho_i$ in probability;
	\item $C_2(G_n) = o(n)$ with high probability.
\end{itemize}
The expected size of the connected component of a vertex $i$ is
\begin{equation}
c_i = \rho_i \EE{C_1(G)} + o(n)~.
\end{equation}
Proposition \ref{prop:supercrit} follows from the following lemmas for the SBM and the \chunglu models.

\begin{lemma}[Coordinate order preserving in the SBM]
	\label{coordinate_order_sbm}
	Assume the conditions of Proposition \ref{prop:supercrit} and let $i_* = \argmax_i b_i$.
	Let $a=(a_1, \dots, a_S)$ be such that $a_j \in [0, a_{i_*}]$ for all $j$. Then $\pa{\Phi_A(a)}_{i_*} \ge \pa{\Phi_A(a)}_{j}$.
\end{lemma}
\begin{proof}
	Let us fix two arbitrary indices $i$ and $i'$. By the definition of $\Phi_M$, we have
	\begin{align*}
	\pa{\Phi_A(a)}_i &= 1 - e^{-((\sum_{j \neq i} \alpha_j a_j)k + \alpha_i k_{i,i} a_i)}~,\\
	\pa{\Phi_A(a)}_{i'} &= 1 - e^{-((\sum_{j \neq i'} \alpha_j a_j)k + \alpha_{i'} k_{i',i'} a_{i'})}~.
	\end{align*}
	Notice that if $i$ and $i'$  satisfy
	$$\pa{\sum_{j \neq i} \alpha_j a_j}k + \alpha_i k_{i,i} a_i \ge 
	\pa{\sum_{j \neq i'} \alpha_j a_j}k + \alpha_{i'} 
	k_{i',i'} a_{i'},$$
	we have $\pa{\Phi_A(a)}_i \ge \pa{\Phi_A(a)}_{i'}$.
	Now, using the facts that
	\begin{itemize}
		\item $\sum_{j \neq i} \alpha_j a_j - \sum_{j \neq i'} \alpha_j a_j =  
		\alpha_{i'} a_{i'}  - \alpha_i a_i$,
		\item $\alpha_i k_{i,i} \ge\alpha_{i} k$, 
		\item $\alpha_i k_{i,i} +\alpha_{i'} k \ge \alpha_{i'} k_{i', i'} + \alpha_{i} k$ and 
		\item $a_{i} - a_{i'} \ge 0$,
	\end{itemize}
	we can verify that
	
	\begin{eqnarray*}
		\lefteqn{
			\alpha_i k_{i,i} a_i + \alpha_{i'} k a_{i'} - \alpha_{i} k a_i -
			\alpha_{i'} k_{i', i'} a_{i'}  } \\
		& = & (\alpha_i k_{i,i} + \alpha_{i'} k) 
		a_{i'} 
		+ (a_{i} - a_{i'})\alpha_i k_{i,i} - (\alpha_{i'} k_{i', i'} +
		\alpha_{i} k)a_{i'} - (a_{i} - a_{i'})\alpha_{i} k \ge 0,
	\end{eqnarray*}
	thus proving the lemma.
\end{proof}
\begin{lemma}[Order of coordinates of eigenvector in the SBM]
	\label{eigenvector_order_sbm}
	Let $a$ be the eigenvector corresponding to the largest eigenvalue
	$\lambda$ of the matrix $M = K \mbox{diag}(\alpha)$. Then if  $i_* = \argmax_m b_m$, we have 
	$a_{i_*} \ge a_j$ for $j \neq i_*$.
\end{lemma}
\begin{proof}
	If $a$ is an eigenvector of $M$, then for coordinates $i, i'$:
	
	$$\begin{cases} \pa{\sum_{j \neq i} \alpha_j a_j}k + \alpha_i k_{i,i} a_i = \lambda a_i, \\
	\pa{\sum_{j \neq i'} \alpha_j a_j}k + \alpha_i k_{i',i'} a_{i'} = \lambda a_{i'}  \end{cases}$$
	By the Perron-Frobenius theorem and our conditions on matrix $M$, $\lambda$ is a real number larger than one.
	Denote $C = k \sum_{j \neq i, j \neq i'}  \alpha_j a_j$, $x = a_i$, $y = a_{i'}$, $a =  \alpha_i k_{i,i}$, $b = \alpha_{i'} k$, $c = 
	\alpha_{i} k$, $d = \alpha_{i'} k_{i', i'}$. Then,
	\begin{equation}\label{silly_sys}
	\begin{cases} C+ ax +by = \lambda x, \\
	C + cx + dy = \lambda y  \end{cases}
	\end{equation}
	Let $r = 1 + \epsilon$ be such that $y = rx = (1 + \epsilon)x$. Then
	$$\begin{cases} \frac{C}{x}+ a +b + b\epsilon = \lambda,  \\
	\frac{C}{x} + c + d + d\epsilon = \lambda + \lambda\epsilon  \end{cases}$$
	and therefore
	$$\frac{C}{x} + c + d + d\epsilon = \frac{C}{x}+ a +b + b\epsilon + \lambda\epsilon~. $$
	Rearranging the terms and using the fact that $a+b \ge c + d$, we have
	$$
	0 \le (a+b) - (c+d) = (d-b-\lambda)\epsilon~.
	$$ 
	Since $k_{i,i} \ge k$, we have $\alpha_i k_{i,i} \ge \alpha_i  k$ and $a \ge c$.
	
	We consider two cases separately: First, if $b \ge d$, we have $d - b - \lambda < 0$, which implies $\epsilon < 0$ and $y < x$, therefore 
	proving $a_i > a_{i'}$ for this case. In the case when $b < d$, we have
	$a+b \ge c + d$ and $\frac{d-b}{a-c} \le 1$. Subtracting the two equalities of the linear system \ref{silly_sys}, we get
	$$\lambda(1-r) = (a-c)\pa{1 - \frac{d-b}{a-c} r}~.$$ 
	Now, since  $\frac{d-b}{a-c} \le 1$, we have $\lambda \ge a - c$, which implies $\lambda \ge d -b$ and $d - b - \lambda \le 0$, thus leading to
	$\epsilon \le 0$ and  $y \le x$, therefore proving  $a_i \ge a_{i'}$ for this case. 
\end{proof}

\begin{lemma}[Order of coordinates of eigenvector in the \chunglu model]
	\label{eigenvector_order_chlu}
	Let $a$ be the eigenvector corresponding to the largest eigenvalue
	$\lambda$ of the matrix $A$. Then if  $i_* = \argmax_m b_m$, we have 
	$a_{i_*} \ge a_j$ for $j \neq i_*$.
\end{lemma}
\begin{proof}
	It is easy to see that the only eigenvector of $A$ corresponding to a non-zero eigenvalue is $a=w$ with $\lambda_{max} = 
	w\transpose w/n$: 
	$$A w = \frac{1}{n}\cdot (ww\transpose)w = \frac{w\transpose w}{n}\cdot w.$$ 
	The proof is concluded by observing that the maximum coordinate of the vector 
	$b$ corresponds to the maximum coordinate of $w$, due to the equality
	\[b_i = \frac{1}{n}\cdot w_i \sum_{j=1}^n w_j.\]
\end{proof}

\begin{lemma}[Coordinate order preserving in the \chunglu model]
	\label{coordinate_order_chlu}
	Assume the conditions of Proposition \ref{prop:supercrit} and let $i_* = \argmax_i b_i$.
	Let $a=(a_1, \dots, a_n)$ be such that $a_j \in [0, a_{i_*}]$ for all $j$. Then $\pa{\Phi_A(a)}_{i_*} \ge \pa{\Phi_A(a)}_{j}$.
\end{lemma}
\begin{proof}
	Let us fix two arbitrary indices $i$ and $i'$. By the definition of $\Phi_A$, we have
	\begin{align*}
	\pa{\Phi_A(a)}_i &= 1 - e^{-w_i(\sum_{j = 1}^n w_j a_j)}~,\\
	\pa{\Phi_A(a)}_{i'} &= 1 - e^{-w_{i'}(\sum_{j = 1}^n w_j a_j)}~.
	\end{align*}
	Then we have $\pa{\Phi_A(a)}_i \ge \pa{\Phi_A(a)}_{i'}$
	thus proving the lemma.
\end{proof}

We finally study the maximal fixed point of the operator $\Phi_A$, keeping in mind this fixed point is exactly the survival-probability 
vector $\rho$ of the multi-type Galton--Watson branching process \cite{Bollobas:2007:PTI:1276871.1276872}. By Lemma~5.9 of 
\cite{Bollobas:2007:PTI:1276871.1276872}, this is the unique fixed point satisfying $\rho_i > 0$ for all $i$. The following lemma shows 
that $\rho_i$ takes its maximum at $i_* = \argmax_i b_i$, concluding the proof of Proposition~\ref{prop:supercrit}.
\begin{lemma}[Fixed point coordinate domination]
	\label{lma_mean}
	Let $\rho$ be the unique non-zero fixed point of $\Phi_A$, and let $i_* = \argmax_i b_i$. Then, 
	$\rho_{i_*} \ge \rho_j$ and $\rho_{i_*} - \rho_j \le \rho^* \pa{b_{i_*} - b_j}$ holds for all $j\neq i_*$.
\end{lemma} 
\begin{proof}
	Letting $a$ be the eigenvector of $A$ that corresponds to the largest eigenvalue $\lambda$, Lemma~\ref{eigenvector_order_chlu}, \ref{eigenvector_order_sbm} 
	guarantees $a_{i_*} \ge a_j$ for  $j\neq i^*$. 
	Let $\epsilon > 0$ be such that $\epsilon \le \frac{1 - 1/\lambda}{a^*}$, where $a^* = \max_{i=1,\ldots,S}a_i$. Then by Lemma~5.13 
	of \cite{Bollobas:2007:PTI:1276871.1276872}, $\Phi_M(\epsilon a) \ge \epsilon a$ holds elementwise for the two vectors. 
	
	Since the coordinates of the vector $\epsilon a$ are positive, we can appeal to Lemma 5.12
	of \cite{Bollobas:2007:PTI:1276871.1276872} to show that iterative application of $\Phi_A$ converges to the fixed point $\rho$:
	letting $\Phi_A^{m}$ be the operator obtained by iterative application of $\Phi_A$ for $m$ times, we have
	$\lim_{m\ra \infty} \Phi_A^{m}(\epsilon a) = \rho$, where $\rho$ satisfies $\rho \ge \epsilon a \ge 0$ and $\Phi_A (\rho) = \rho > 
	0$.
	By the respective Lemmas~\ref{eigenvector_order_chlu}, \ref{eigenvector_order_sbm}  we have $\rho_{i_*} \ge \rho_j$, for $i_* \neq j$ 
	for both the SBM and the Chung--Lu models, proving the first statement.
	
	The second statement can now be proven directly as
	\begin{eqnarray*}
		\lefteqn{
			\rho_{i_*} - \rho_i = e^{-(A\rho)_{j}} - e^{-(A\rho)_{i_*}} = e^{- \sum_{j}^{n} A_{i_* j}\rho_j} - e^{- \sum_{j}^{n} A_{ij}\rho_j}   
		} \\
		& = & e^{- \sum_{j}^{n} A_{i_* j}\rho_j} (1 - e^{- \sum_{j}^{n} A_{ij}\rho_j- A_{i_* j}\rho_j}  ) \le 
		e^{- \sum_{j}^{n} A_{i_* j}\rho_j}  \left( \sum_{j}^{n} (A_{i_* j} - A_{i j})\rho_{i_*}  \right) 
		\\
		& \le & \rho^* (b_{i_*} - b_{i}),
	\end{eqnarray*}
	where the first inequality uses the relation $1-e^{-z} \le z$ that holds for all $z\in\real$, and the last step uses the fact that 
	$A\rho$ has positive elements.
\end{proof}

\section{The proof of Theorem~\ref{thm:n_klucb}}\label{app:klucb}
Before delving into the proof, we introduce some useful notation. We start by defining $Y_{a,1}, 
Y_{a,2}, \dots, Y_{a,n}$ as independent  Bernoulli  random  variables with 
respective parameters $\mathbb{B} = \pa{A_{a,1}, A_{a,2}, \dots , A_{a,n}}$, and noticing that the degree $Y_{t,a}$ can be written as a sum 
$Y_a = 
\sum_{i\neq a} Y_{a,i}$. The following lemma, used several times in our proofs, relates this quantity to a Poisson distribution with the 
same mean.
\begin{lemma}\label{lem:domination}
	Let $i\in[S]$ and let $Y_{i,1}, Y_{i,2}, \dots, Y_{i,n}$ be independent Bernoulli random variables with respective parameters $k_{i,1}/n, 
	k_{i,2}/n, \dots , k_{i,n}/n$, and let $X_i$ be a Poisson random 
	variable with parameter $\mu_i = \sum_{j \neq i} k_{i,j}/n $. Defining $Y_i = \sum_{j\neq i} Y_{i,j}  $, we have $\EE{e^{sY_i}} \le 
	\EE{e^{sX_i}}$ for all $s\in\real$.
\end{lemma}
\begin{proof}
	Fix an arbitrary $s\in\real$ and $i\in[n]$. By direct calculations, we obtain
	\begin{align*}
	\mathbb{E} e^{s Y_i} 
	&= \prod_{j=1}^{n} \left(\mathbb{E} e^{s Y_{i,j}}\right) \le \prod_{j=1}^{n} \left(1 + \frac{k_{i,j}}{n}(e^s - 1)\right)
	\le  \prod_{j=1}^{n}  \exp\pa{\pa{k_{i,j}/n} \cdot  (e^s - 1)} ,
	\end{align*}
	where the last step follows from the elementary inequality $1+x \le e^{x}$ that holds for all $x\in\real$. The proof is concluded by 
	observing that $\mathbb{E} e^{s X_i} = \exp\pa{\mu\pa{e^{s}-1}}$ and using the definition of $\mu$.
\end{proof}
For simplicity, we also introduce the notation $\psi_{\mathbb{B} }(s) = \log \E{e^{sY_i}}$ and
$\phi_{\lambda}(s) = \log \mathbb{E} e^{s X_i} = \lambda (e^s - 1)$. 
The proof below repeatedly refers to the Fenchel conjugate of 
$\phi_\lambda$ defined as
\begin{align*}
\phi_{\lambda}^*(z) = \sup_{s \in \real} \{sz - \phi(s)\} = z \log \left(\frac{z}{\lambda}\right) + \lambda - z 
\end{align*}
for all $z\in\real$. Finally, we define
$d(\mu, \mu') = \mu' - \mu + \mu \log\left(\frac{\mu}{\mu'}\right)$ for all $\mu,\mu' > 0$, noting that $\phi_{\lambda}^*(z) = 
d(z,\lambda)$.

As for the actual proof of the theorem, the statement is proven in four steps. Within this proof, we refer to nodes as \emph{arms} and use 
$K$ to denote the size of 
$V_0$. We use the notation $f(t) = 3\log t$.
\paragraph{Step 1.} We begin by rewriting the expected number of draws $\E{N_a}$ for any suboptimal arm $a$ as 
\[
\mathbb{E}N_a = \EE{\sum_{t = K}^{T-1} \mathbb{I} \{A_{t+1} = a  \}} = \sum_{t = K}^{T-1} \mathbb{P} \{A_{t+1} = a  \}.
\]
By definition of our algorithm, at rounds $t > K$, we have $A_{t+1} = a$ only if $U_{a(t)} >   U_{a^*(t)}$. This leads to the 
decomposition:
\begin{align*}
\{A_{t+1} = a \} &\subseteq \{\mu^* \ge U_{a^*}(t)  \} \cup \{\mu^* < U_{a^*}(t) \text{ and } A_{t+1} = a \}
\\
&\subseteq \{\mu^* \ge U_{a^*}(t) 
\} \cup \{\mu^* < U_{a}(t) \text{ and } A_{t+1} = a \} 
\end{align*}
Steps~2 and~3 are devoted to bounding the probability of the two events above.

\paragraph{Step 2.} Here we aim to upper bound
\begin{equation}\label{eq:underest}
\sum_{t = K}^{T-1} \PP{\mu^* \ge U_{a^*}(t)}.
\end{equation}
Note, that $\left\{ U_{a^*}(t) \le \mu^*\right\} = \left\{\hat{\mu}_{a^*}(t) \le U_{a^*}(t) \le \mu^*\right\}.$
Since $d(\mu, \mu') = \mu' - \mu + \mu\log(\frac{\mu}{\mu'})$ is non-decreasing in its second argument on $[\mu, + \infty)$, and by 
definition of  $U_{a^*}  = \sup \{ \mu: d( \hat{\mu}_{a^*}(t), \mu) \le \frac{f(t)}{N_{a^*(t)}} \}  $ we have
\[
\left\{\mu^* \ge U_{a^*}(t) \right\} \subseteq \left\{ \hat{\mu}_{a^*}(t) \le U_{a^*}(t) \le \mu^* \text{ and }  d(\hat{\mu}_{a^*}(t), 
\mu^*) \ge 
\frac{f(t)}{N_{a^*}(t)} \right\},
\]
Taking a union bound over the possible values of $N_{a^*}(t)$ yields 
\begin{align*}\label{eq:arm-decomposition}
\left\{\mu^* \ge U_{a^*}(t) \right\} \subseteq  \bigcup_{n=1}^{t - K + 1} \left\{ \mu^* \ge \hat{\mu}_{a^*, n} \text{ and } 
d(\hat{\mu}_{a^*, n}, \mu^*) \ge \frac{f(t)}{n} \right\} =  \bigcup_{n=1}^{t - K + 1} D_n(t),
\end{align*}
where the event $D_n(t)$ is defined through the last step.
Since $d(\mu, \mu^*)$ is decreasing and continuous in its first argument on  $[0, \mu^*)$, either $d(\hat{\mu}_{a^*, n}, \mu^*) < 
\frac{f(t)}{n}$ on this interval and $D_n(t)$ is the empty set, or there exists a unique $z_n \in [0, \mu^*)$ such that $d(z_n, \mu^*) =  
\frac{f(t)}{n}$. Thus, we have
\[
\bigcup_{n=1}^{t - K + 1} D_n(t) \subseteq  \bigcup_{n=1}^{t - K + 1} \left\{ \hat{\mu}_{a^*, n} \le  z_n \right\}.
\]
For $\lambda < 0$, let us define $\psi(\lambda)$ as the cumulant-generating function of the sum of binomials with 
parameters $\mathbb{B} $, and let $\phi(\lambda)$ be the cumulant-generating function of a Poisson random variable with parameter 
$\mu^*$. With this notation, we have for \emph{any} $\lambda < 0$ that
\begin{align*}
\PP{ \hat{\mu}_{a^*, n} \le  z_n } &=   \PP{ \exp(\lambda \hat{\mu}_{a^*, n}) \ge  \exp(\lambda  z_n) } 
\\
&= \PP{ \exp\pa{\lambda \sum_{i = 1}^n Y_{a^*, i} - n\psi(\lambda)} \ge  \exp(n \lambda  z_n - n\psi(\lambda))} 
\\
&\le \left(\frac{\mathbb{E} e^{\lambda Y_{a^*,1}}}{e^{\psi(\lambda)}}\right)^n e^{-n(\lambda z_n - \psi(\lambda))}  
\le  e^{-n(\lambda z_n - \psi(\lambda))},
\end{align*}
where the last step uses the definition of $\psi(\lambda)$. Now fixing $\lambda^* = \argmax_{\lambda} \{ \lambda z_n - \phi(\lambda) \}  = 
\log(z_n/\mu^*) < 0$, we get by Lemma \ref{lem:domination} that
\begin{align*}
e^{-n(\lambda^* z_n - \psi(\lambda^*))} \le e^{-n(\lambda^* z_n - \phi(\lambda^*))}  = e^{-n \phi^*_{\mu^*}(z_n)} = e^{ - n d(z_n, 
	\mu^*)}~. 
\end{align*}
In view of the definition of $z_n$ and $f(t)$, this gives the bound
\begin{align*}
e^{ - n d(z_n, \mu^*)}  =  e^{- f(t)} = \frac{1}{t^3},
\end{align*}
which leads to
\begin{align*}
\sum_{t = K}^{T-1} \PP{\mu^* \ge U_{a^*}(t)} \le  \sum_{t = K}^{T-1}  \sum_{n=1}^{t - K + 1} \frac{1}{t^3} < 2,
\end{align*}
thus concluding this step.

\paragraph{Step 3.} 
In this step, we borrow some ideas by \cite[Proof Theorem~2, step~2]{2017arXiv170207211M}  to upper-bound the sum
\begin{equation} \label{sum_3rd}
B = \sum_{t = K}^{T-1} \PP{\mu^* < U_{a}(t) \text{ and } A_{t+1} = a }.
\end{equation}
Writing $\eta = \eta_a = \{\mu^* - \mu_a\}/3$ for ease of notation, we have
\begin{align*} 
\{\mu^* < U_{a}(t) \text{ and } A_{t+1} = a \} &\subseteq \ev{\mu^* - \eta < U_{a}(t) \text{ and } A_{t+1} = a} 
\\
&\subseteq 
\ev{d(\hat{\mu}_{a}(t), \mu^* - \eta ) \le f(t)/N_a(t) \text{ and } A_{t+1} = a}.
\end{align*}
Thus, we have
\begin{align*} 
B &\le \sum_{t = K}^{T-1} \PP{d(\hat{\mu}_{a}(t), \mu^* - \eta ) \le f(t)/N_a(t) \text{ and } A_{t+1} = a}
\\
&\le \sum_{n = 1}^{T} \PP{ d(\hat{\mu}_{a,n}, \mu^* - \eta ) \le f(T)/n }
\end{align*}
Defining the integer $n(\eta)$ as
$$
n(\eta) = \left \lceil \frac{f(T)}{d(\mu_{a} + \eta, \mu^* - \eta)} \right \rceil,
$$
we have $f(T)/n \le d(\mu_{a} + \eta, \mu^* - \eta)$ for all $n \ge n(\eta) $. Thus, we may further upper-bound $B$ as
\begin{align*}
B &\le n(\eta) - 1 + \sum_{n = n(\eta)}^{T} \PP{d(\hat{\mu}_{a,n}, \mu^* - \eta ) \le f(T)/n} 
\\
&\le \frac{f(T)}{d(\mu_{a} + \eta, \mu^* - \eta)}  +   \sum_{n = n(\eta)}^{T} \PP{d(\hat{\mu}_{a,n}, \mu^* - \eta ) \le d(\mu_{a} + 
	\eta, \mu^* - \eta)}. 
\end{align*}
By definition of $\eta$, we have
$$\ev{\hat{\mu}_{a,n}, \mu^* - \eta ) \le d(\mu_{a} + \eta, \mu^* - \eta)} \subseteq \ev{\hat{\mu}_{a,n} \ge \mu_a + \eta},$$
which implies
\begin{align*}
\sum_{n = n(\eta)}^{T}  \PP{d(\hat{\mu}_{a,n}, \mu^* - \eta ) \le d(\mu_{a} + \eta, \mu^* - \eta)} \le \sum_{n = n(\eta)}^{T} 
\PP{\hat{\mu}_{a,n} \ge \mu_a + \eta}.
\end{align*}
By an argument analogous to the one used in the previous step, we get for a well-chosen $\lambda$ that
\begin{align*}
&\sum_{n = n(\eta)}^{T} \PP{\hat{\mu}_{a,n} \ge \mu_a + \eta} 
\le \PP{\exp(\lambda \hat{\mu}_{a,n}) \ge \exp(\lambda(\mu_a+\eta))} 
\\
&\qquad\qquad\qquad= \sum_{n = n(\eta)}^{T} \PP{\exp(\lambda\sum_{i = 1}^n Y_{a, i} - n\psi(\lambda)) \ge \exp(n\lambda(\mu_a + \eta) 
	- n\psi(\lambda)) } \\
&\qquad\qquad\qquad\le \sum_{n = n(\eta)}^{T} \pa{\frac{\EE{e^{\lambda Y_{a,i}}}}{e^{\psi(\lambda)}}}^n e^{-n(\lambda(\mu_a+\eta) - 
	\psi(\lambda))} 
\\
&\qquad\qquad\qquad
\le \sum_{n = n(\eta)}^{T} e^{-n(\lambda(\mu_a+\eta) - \phi(\lambda))} 
= \sum_{n = n(\eta)}^{T} e^{-n d(\mu_a+\eta, \mu_a)} 
\\
&\qquad\qquad\qquad\le 
\sum_{n = n(\eta)}^{ \infty}  e^{-n d(\mu_a+\eta, 
	\mu_a)} \le \frac{1}{e^{d(\mu_a+\eta, \mu_a)} - 1} \le \frac{1}{d(\mu_a+\eta, \mu_a)},
\end{align*}
where the last step uses the elementary inequality $1+x \le e^{x}$ that holds for all $x\in \real$. 
\paragraph{Step 4.}
Putting together the results from the first three steps, we get
\begin{align*}
\mathbb{E} N_{a} \le  3  +  \frac{1}{d(\mu_a+\eta, \mu_a)} +  \frac{3 \log T}{d(\mu_{a} + \eta, \mu^* - \eta)}.
\end{align*}
We conclude by taking a second-order Taylor-expansion of $d(\mu_a + \eta,\mu_a)$ in $\eta$ to obtain for some  $\eta' \in [0,\eta]$
that
$$
d(\mu_a+\eta, \mu_a) = \frac{\eta^2}{2(\mu_a+\eta')} \ge \frac{\eta^2}{2(\mu_a+\eta)}.
$$
Taking into account the definition of $\eta$, we get
$$\frac{1}{d(\mu_a+\eta, \mu_a)} \le  \frac{2\mu^*}{\eta^2}.$$
An identical argument can be used to bound $\pa{d(\mu_{a} + \eta, \mu^* - \eta)}^{-1} \le 2\mu^*/\eta^2$.
\qed

\section{The proof of Theorem~\ref{thm:doubling}}\label{app:doubling}
We start by assuming that $\alpha < 1/2$.
Also notice that for a uniformly sampled set of nodes $U$, the probability of $U$ not containing a vertex from $V^*_{\alpha} $ is 
bounded as
\[
\PP{U\cap V^*_{\alpha}  = \emptyset} \le (1 - \alpha)^{|U|}.
\]
By the definition of $V_k$, this gives that the probability of not having sampled a  node from $V^*_{\alpha} $ in period $k$ of the 
algorithm is bounded as
\[
\PP{V_k \cap V^*_{\alpha} =\emptyset} \le (1 - \alpha)^{|V_k|} \le \beta^{-k}.
\]

For each period $k$, the expected regret can bounded as the weighted sum of two terms: the expected regret of \ducb$(V_k)$ in period $k$ 
whenever $V_k\cap V^*_{\alpha}$ is not empty, and the trivial bound $\Delta_{\alpha,\max} \beta^k$ in the complementary case. Using the 
above 
bound on the 
probability of this event and appealing to Theorem~\ref{thm:n_klucb} to bound the regret of \ducb$(V_k)$, we can bound the expected regret 
as
\begin{align*}
\EE{R^{\alpha}_T} &\le \sum_{k=1}^{k_{\max}} \pa{ \beta^k \frac{1}{\beta^k} \Delta_{\alpha, \max} +  \sum_{i \in V_k} \Delta_{\alpha,i} 
	\pa{ \frac{\mu^*\pa{2+3\log \beta^k}}{\delta_{\alpha,i}^2} +  3 }} 
\\
&\le k_{\max}  \Delta_{\alpha, \max}  + \sum_{k=1}^{k_{\max}}  \pa{  \sum_{i \in V_k}  \Delta_{\alpha,i} 
	\pa{ \frac{\mu^*\pa{2+3k \log \beta}}{\delta_{\alpha,i}^2} +  3 }} 
\\
&\le k_{\max}  \Delta_{\alpha, \max} +     \sum_{i \in \overline{V}}  \Delta_{\alpha,i}  \pa{  \pa{3 +  
		\frac{2\mu^*}{\delta_{\alpha,i}^2} } (k_{\max} + 1) +  \frac{3 \log \beta (k_{\max}+1)^2 }{ 2\delta_{\alpha,i}^2}}~.
\end{align*}
The proof of the first statement is concluded by upper-bounding the number of restarts up to time $T$ as
$k_{\max} \le \frac{\log T}{\log \beta}$.

The second statement is proven by an argument analogous to the one used in the proof of Theorem~\ref{thm:reg_knowT}, and straightforward 
calculations.
\qed

\end{document}